%% file: arXiv.tex
\documentclass{article}

\usepackage{arxiv}

\usepackage[american]{babel}

\input{def.tex}

\usepackage[utf8]{inputenc} % allow utf-8 input
\usepackage[T1]{fontenc}    % use 8-bit T1 fonts
\usepackage{hyperref}       % hyperlinks
\usepackage{url}            % simple URL typesetting
\usepackage{booktabs}       % professional-quality tables
\usepackage{amsfonts}       % blackboard math symbols
\usepackage{nicefrac}       % compact symbols for 1/2, etc.
\usepackage{microtype}      % microtypography
\usepackage{xcolor}         % colors

\usepackage{algorithm}
\usepackage{algorithmic}
\usepackage{multirow}

\newcommand\independent{\protect\mathpalette{\protect\independenT}{\perp}}
\def\independenT#1#2{\mathrel{\rlap{$#1#2$}\mkern2mu{#1#2}}}

\usepackage{microtype}
\usepackage{graphicx}
\usepackage{subfigure}
\usepackage{booktabs} % for professional tables

\usepackage{hyperref}

\usepackage{amsmath}
\usepackage{amssymb}
\usepackage{mathtools}
\usepackage{amsthm}
\usepackage{natbib}

\usepackage[capitalize,noabbrev]{cleveref}

\usepackage{graphicx}

\theoremstyle{plain}
\newtheorem{theorem}{Theorem}[section]

\newtheorem{lemma}[theorem]{Lemma}

\theoremstyle{definition}

\theoremstyle{remark}

\usepackage{color}

\allowdisplaybreaks

\title{Automatic Debiased Learning\\
from Positive, Unlabeled, and Exposure Data}

\usepackage{authblk}
\usepackage{subfigure}
\graphicspath{ {./Figures/} }

\begin{document}

\author{Masahiro Kato}
\author{Shuting Wu}
\author{Kodai Kureishi}
\author{Shota Yasui}

\affil{AI Lab, CyberAgent, Inc.}

\maketitle

\begin{abstract}
We address the issue of binary classification from positive and unlabeled data (PU classification) with a selection bias in the positive data. During the observation process, (i) a sample is exposed to a user, (ii) the user then returns the label for the exposed sample, and (iii) we however can only observe the positive samples. Therefore, the positive labels that we observe are a combination of both the exposure and the labeling, which creates a selection bias problem for the observed positive samples. This scenario represents a conceptual framework for many practical applications, such as recommender systems, which we refer to as ``learning from positive, unlabeled, and exposure data'' (PUE classification). To tackle this problem, we initially assume access to data with exposure labels. Then, we propose a method to identify the function of interest using a strong ignorability assumption and develop an ``Automatic Debiased PUE'' (ADPUE) learning method. This algorithm directly debiases the selection bias without requiring intermediate estimates, such as the propensity score, which is necessary for other learning methods. Through experiments, we demonstrate that our approach outperforms traditional PU learning methods on various semi-synthetic datasets.
\end{abstract}

\section{Introduction}
In many applications, collecting a large amount of labeled data is often costly, and classification algorithms often utilize data whose labels are imperfectly observed. Such an algorithm, known as weakly supervised classification, has gained increasing attention in recent years \citep{Sugiyama2022}. In this study, we address the problem of learning a binary classifier from positive and unlabeled data (PU classification) under a \emph{selection bias}. Specifically, we examine the following labeling mechanism: a sample is exposed to a user, the user labels the exposed sample, and only the positive data labeled by the user are observed (See Figure~\ref{fig:difference2}). In this setting, a selection bias arises as the labeled positive data do not correspond to pure positive labeling events, but rather to the joint events of exposure and positive labeling. To address this issue, in addition to PU data, we assume access to data with comprising pairs of exposure labels and feature random variables. We then propose a learning method that debiases the selection bias from a classifier trained using the observations. We refer to this problem as PU classification with a selection bias and Exposure labels (PUE classification).

In real-world scenarios, such as post-click conversion in online advertising, predicting users' preferences from observations is a crucial task. However, due to privacy concerns, it has become increasingly difficult to use identifiers such as the ``Identifier for Advertisers\footnote{A device identifier of Apple, which is assigns to every device.}''. As a result, advertisement exposure and user consumption data are often stored separately. Consequently, we need post-click conversion prediction models trained from such separate data.

We consider learning algorithms for the PUE classification problem. Under the assumption of strong ignorability \citep{Rosenbaum1983}, which implies that the labeling of positive and negative and exposure are independent given a feature of a sample, we propose an identification strategy for a classifier that predicts the true labels. Under the assumption and identification strategy, we develop a method that automatically debiases the selection bias and returns an estimator of the conditional probability of the true label. Finally, through experimental studies, we confirm the soundness of our proposed method. 

%ここから
In the context of PU classification, there exist two distinct sampling schemes \citep{elkan2008learning}, one-sample and two-sample settings. Besides, there are two labeling mechanisms \citep{Bekker2020}, selection-completely-at-random (SCAR) and selection-at-random (SAR). 
In this study, we consider both one-sample and two-sample settings, incorporating an additional dataset comprising of exposure labels. Furthermore, we consider SAR with a strong ignorability assumption.
In the one-sample setting, existing studies propose utilizing the inverse propensity score (IPS) method \citep{Horvitz1952} \citep{Bekker2018,Bekker2019}, commonly utilized in missing value imputations. In the two-sample setting, partial identification \citep{Manski2008} has been applied with a suitable condition \citep{kato2018pubp}. Our proposed PUE classification approach is related to one-sample and two-sample PU classification problems with a selection bias, however, it differs from existing methods under differing assumptions. In particular, our proposed method does not necessitate intermediate estimates such as the propensity score in the IPS method \citep{Bekker2018,Bekker2019}, and returns the optimal classifier.
%ここまで

In addition to the PUE classification, we present several extensions with methods for variant settings. These settings and methods can be extended to various scenarios, including semi-supervised classification. Our formulation of semi-supervised classification is based on the problem of classifying positive, negative, and unlabeled data \citep[PNU classification,][]{Sakai2017}, which is an extension of methods for PU classification to semi-supervised classification \citep{duPlessis2015,Gang2016}.

The problem of PUE classification and its variants arises in various practical situations, not only in online advertising. For example, in recommender systems, one of the key tasks is predicting users' preferences from implicit feedback, such as users' clicks (e.g., purchases, views) \citep{Jannach2018, Jiawei2018, Yifan2008, Liang2016}. While collecting explicit feedback, such as ratings, is costly, it is easy to collect implicit feedback, and users' behavior logs can be considered as them \cite{Liu2019}. In implicit feedback, we observe users' positive preferences for items through their clicks. However, we cannot observe negative preferences. A user's non-click on an item does not necessarily mean that the user dislikes the item; perhaps the item was not exposed to the user. Additionally, there can be various other real-world applications, such as information retrieval \citep{Joachims2016,Joachims2017,Xuanhui2918}, inlier-based outlier detection, and crowdsourcing.

The contributions of this study are summarized follows:
\vspace{-0.5\baselineskip}
\begin{itemize}
    \vspace{-0.5\baselineskip}
    \item problem formulation of PUE classification with a selection bias by using the potential outcome framework.
    \vspace{-0.5\baselineskip}
    \item an algorithm, called ADPUE learning, which automatically debias the bias in a classifier.
    \vspace{-0.5\baselineskip}
    \item theoretical justification of our proposed algorithm. 
    \vspace{-0.5\baselineskip}
    \item other related problem formulations, including semi-supervised learning.
    \vspace{-0.5\baselineskip}
    \item empirical analysis of proposed methods by simulations. 
    \vspace{-0.5\baselineskip}
\end{itemize}
\vspace{-0.5\baselineskip}

\section{Problem Setting}
For each sample $i = 1,2,\dots, n$, let $\bm{x}_i \in\mathcal{X}\subset \mathbb{R}^d$ be a feature variable with its space $\mathcal{X}$ and $y_i\in\{1, 0\}$ be a potential random label, where $y_i = 1$ and $y_i = 0$ indicate positive and negative labels, respectively. Our goal is to classify $\bm{x}_i \in \mathcal{X} \subset \mathbb{R}^d$ into one of the two classes $\{1, 0\}$. 

\subsection{Potential Labels}
Because of the selection bias, we cannot observe $y_i$ for all samples. We formulate our data-generating process (DGP) employing the potential outcome framework from \citet{imbens_rubin_2015}. 
For each $i$, we define a potential random variable $(Y_i, X_i)$, where  $Y_i \in \{1, 0\}$ is a potential binary label, and $X_i \in \mathcal{X}$ is a feature random variable. 
The label variable is potential in the sense that it exists even if it is not observed. Then, for each $i$, we assume the following DGP:
\begin{align}
    \big(Y_i, X_i\big) \iid p\big(y|\bm{x}\big)\zeta\big(\bm{x}\big),
\end{align}
where $p\big(y| \bm{x}\big)$ is the conditional density of $Y_i \in \{1, 0\}$ given $X_i = \bm{x}\in\mathcal{X}$, and $\zeta(\bm{x})$ is the density of $\bm{x}$. 

\subsection{Observations}
\label{sec:obs}
In PUE classification, we cannot observe the potential labels $(Y_i)$ directly because of the selection bias. 
Let us denote an exposure event by $E_i \in \{1, 0\}$, where $E_i = 1$ if a sample is exposed to an user, and $E_i = 0$ if not.  Then, we consider the following labeling procedure:
\vspace{-0.5\baselineskip}
\begin{description}
    \vspace{-0.5\baselineskip}
    \item[(i)] a sample $i$ is randomly exposed to an user.
    \vspace{-0.5\baselineskip}
    \item[(ii)] if the sample is exposed $(E_i = 1)$, the user labels the sample ($Y_i$).
    \vspace{-0.5\baselineskip}
    \item[(iii)] only a part of positive data ($Y_i = 1$) is observed.
    \vspace{-0.5\baselineskip}
\end{description}
\vspace{-0.5\baselineskip}
Besides, as separately labeled data, we can observe a pair of $(E_i, Y_i)$. In summary, we define the observed data, or equivalently training data, and the DGP as
\begin{align}
    \mathcal{D}^{\mathrm{PU}} &= \left\{\big(W_i, X_i\big)\right\}^{n^{\mathrm{PU}}}_{i=1}\iid q(w| \bm{x})\zeta(\bm{x}),\\
    \mathcal{D}^{\mathrm{E}} &= \left\{\big(E_i, X_i\big)\right\}^{n^{\mathrm{E}}}_{i=1} \iid \theta(e|\bm{x})\zeta(\bm{x}),
\end{align}
where $W_i = E_iY_i\in\{1, 0\}$, $q\big(w|\bm{x}\big)$ is the conditional density of $W_i$ given $X_i$, and $\theta\big(e|\bm{x}\big)$ is that of $E_i$ given $X_i$. Here, $E_i$ is observable. However, because
$\mathcal{D}^{PU}$ and $\mathcal{D}^{E}$ are separately observed, this problem differs from semi-supervised learning setting, and we cannot learn $p(y=1|x, e=1)$ directly. 

\subsection{Population Risk}
Our goal is to obtain a classifier $g:\mathcal{X}\to \{1, 0\}$ from the observed data (training data) to predict $Y_i$ by using $X_i$. In a binary classification, for a class of classifiers $\mathcal{G}$, the optimal classifier $g^*$ is given by $g^* = \argmin_{g\in\mathcal{G}}R^{0\mathchar`-1}(g)$, where $R^{0\mathchar`-1}(g)$ is the expected
misclassification rate (population risk) when the classifier $g(\bm{x})$ is applied to
unlabeled samples distributed according to $\zeta(x)$:
\begin{align*}
&R^{0\mathchar`-1}(g)=\mathbb{E}\big[Y_i \ell^{0\mathchar`-1}(g(X_i), 1) + (1 - Y_i) \ell^{0\mathchar`-1}(g(X_i), 0)\big].
\end{align*}
where $\mathbb{E}$ is the expectation over $(Y_i, X_i)$ with the density $p(y|\bm{x})\zeta(x)$ and $\ell^{0\mathchar`-1}$ is the zero-one loss such that. $\ell^{0\mathchar`-1}(g, z) =  \mathbbm{1}[g \neq z]$. Here, $\mathbbm{1}[\cdot]$ is an indicator function. We call this population risk an \emph{ideal} population risk.

It is known that an optimal classifier is given as $g^*(\bm{x}) = \mathbbm{1}[p(y = 1| \bm{x})\geq 1/2]$; therefore, we only consider classifiers such that $g(\bm{x}) = \mathbbm{1}[f(\bm{x})\geq 1/2]$, where $f(\bm{x})$ is an estimator of $p(y = 1|\bm{x})$. If we redefine the ideal risk using $f$ and a surrogate loss $\ell(f, z)$, an ideal population risk is 
\begin{align*}
&R^*(f)=\mathbb{E}\big[Y_i \ell(f(X_i), 1) + (1 - Y_i) \ell(f(X_i), 0)\big]\\
&=\mathbb{E}\big[p(Y_i = 1| X_i) \ell(f(X_i), 1) + p(Y_i = 0| X_i) \ell(f(X_i), 0)\big].
\end{align*}
We use the log loss as a surrogate loss $\ell(f, z)$; that is, $\ell(f, 1) = \log(f)$ and $\ell(f, 0) = \log(1 - f)$. 
Under the log loss, as we show in Section~\ref{sec:auto_deb_pue}, a minimizer of the population risk can be interpreted as the conditional probability. 

\textbf{PUE classification.}
In PU classification, there are two distinguished sampling schemes called \emph{one-sample} and \emph{two-sample} scenarios \citep{elkan2008learning,Gang2016}. Our setup introduced in this section is more related to the two-sample scenario but is different from it. 
In the standard PU classification, SCAR is traditionally assumed, i.e., the positive labeled data are
identically distributed as the positive unlabeled data \citep{elkan2008learning}.
The assumption of SCAR is, however, unrealistic in many instances of PU learning, e.g., a patient’s
digital health record \citep{Bekker2019} and a recommendation system \citep{Benjamin2009,Schnabel2016,Saito2020}. In these cases, there is a selection bias \citep{Harel79,angrist_mostly_2008}; the distribution of the positive data may differ between the labeled data
and the unlabeled data. Among several ways of selection bias, we specify the procedure of bias as described in this section.
This setting is a special case of the PU classification problem under SAR and the abstraction of several common real-world applications, such as recommender systems and online advertisements. We call our problem the \emph{PUE classification problem}. 
We illustrate this situation in Figure~\ref{fig:difference2}. Our setting and method are general and can be extended to other settings. For example, in Section~\ref{sec:other_form}, we introduce other scenarios that are more closely related to the two-sample scenario. In Table~\ref{table:comparison} of Section~\ref{sec:related}, we summarize our proposed problem settings.  

\begin{figure}[t]\centering
\includegraphics[scale=0.4]{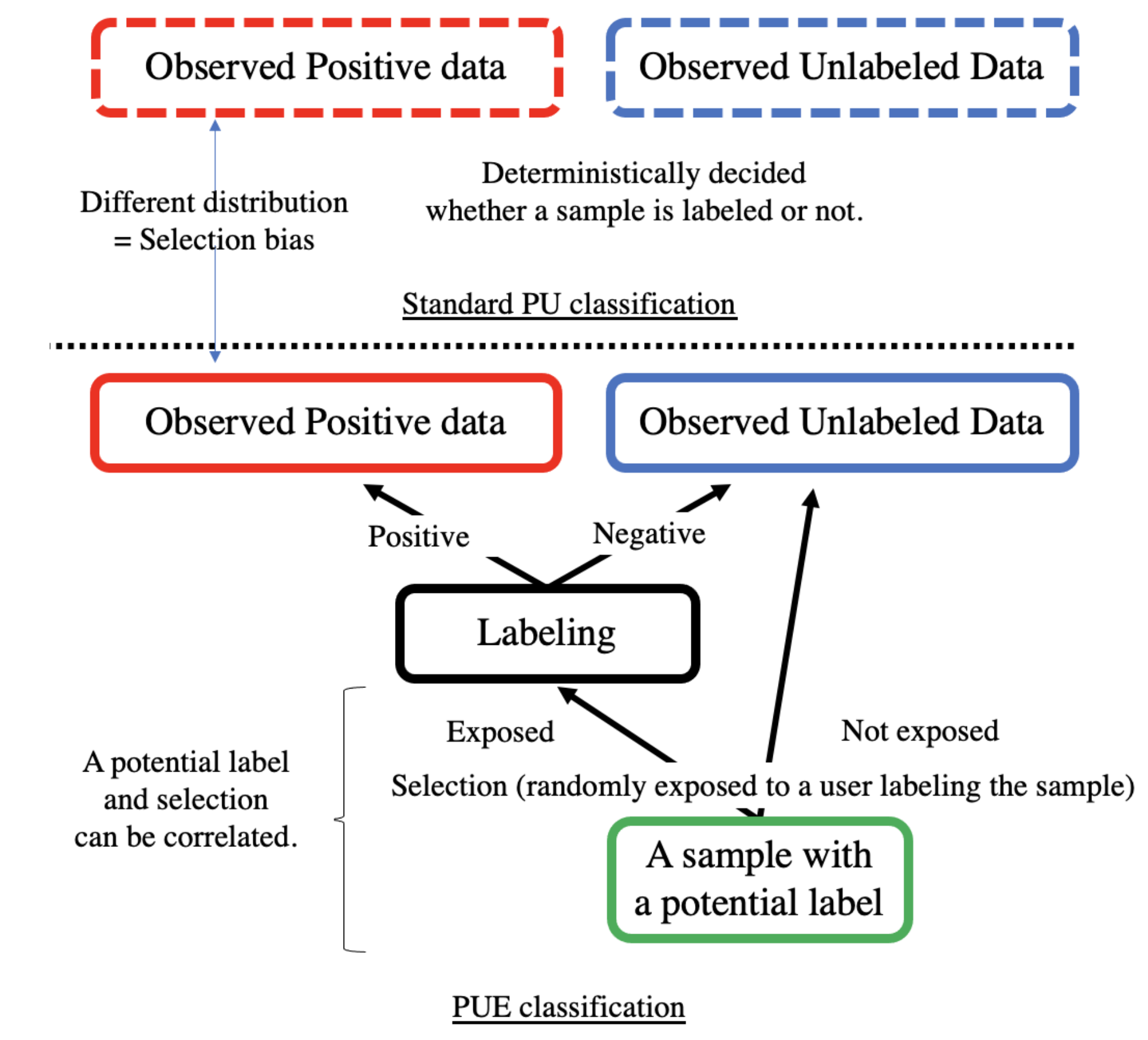}
\vspace{-0.3cm}
\caption{Difference between the standard PU classification and PUE classification.}
\label{fig:difference2}
\vspace{-0.5cm}
\end{figure}

\textbf{Strong ignorability.}
For identification, we maintain throughout the paper the strong ignorability assumption \citep{Rubin1978,Rosenbaum1983}, which asserts that conditioned on the observed covariates, the exposure is independent of the potential outcome, denoted by
\begin{align}
\label{eq:strong1}
    Y_i  \independent E_i | X_i
\end{align}
Under this assumption, we have $p(y = 1, e = 0| \bm{x}) = p(y = 1| \bm{x})p(e = 0| \bm{x})$,
which also implies 
\begin{align}
\label{eq:strong2}
    p(y = 1| e = 0, \bm{x}) = p(y = 1|\bm{x}).
\end{align}
We can directly assume \eqref{eq:strong2}, which is weaker than the latter \eqref{eq:strong1} \citep{Heckman1998}.

\textbf{Test data.}
There are two different scenarios in the definition of test data: \emph{inductive and transductive} settings. In the inductive setting, we  predict labels for general unlabeled test data, $\mathcal{D}^{\mathrm{test}} = \left\{X_i\right\}^{n^{\mathrm{P}}}_{i=1}\iid \zeta(\bm{x})$. In contrast, in the transductive setting, we predict labels for our observed unlabeled data in the training data. In our setting, because we know that all labels of samples in $\mathcal{D}^{\mathrm{P}}$ are $Y_i = 1$, there are two candidates of the target of prediction: $\mathcal{D}^{\mathrm{U}}$ and $\mathcal{D}^{\mathrm{E}}$. In both of the inductive and transductive settings, we predict $Y_i\in\{1, 0\}$ over unlabeled samples with $X_i \iid \zeta(\bm{x})$.

\section{Automatic Debiased PUE (ADPUE) Learning}
\label{sec:auto_deb_pue}
Firstly, we introduce the minimization of the pseudo classification risk, which uses a guess of $p(y = 1| \bm{x})$. Under a selection bias, we cannot construct an unbiased risk function unless the guess is true. However, by repeatedly substituting the estimate of $p(y = 1| \bm{x})$ into the pseudo classification risk as the next guess of $p(y = 1| \bm{x})$ and minimizing the pseudo classification risk, we obtain a sequence of estimates that converge to the true $p(y = 1| \bm{x})$.

\subsection{Identification Strategy}
The central problem is that we can only observe pairs of random variables, $(W_i, X_i)$ and $(E_i, X_i)$, and $Y_i$ is unobservable. To train a classifier of $Y_i$ using $(W_i, X_i)$ and $(E_i, X_i)$, we consider an identification strategy; that is, we investigate under which condition we can train a classifier $p(y| \bm{x})$. 

Let $\theta(y, e| \bm{x}) = p(y|e, \bm{x})\theta(e| \bm{x})$ be the joint density of $(Y_i, E_i)$ conditioned on $X_i$, where $p(y| e, \bm{x})$ is the conditional density of $Y_i$ given $E_i$ and $X_i$.
According to the definition of the conditional density function,  we can identify the density $p(y| \bm{x})$ as follows:
\begin{align*}
    p(y = 1 | \bm{x}) &= \theta(y = 1, e = 1| \bm{x}) + \theta(y = 1, e = 0| \bm{x})\\
    &= q(w = 1| \bm{x}) + p(y = 1| e = 0, \bm{x})\theta(e = 0| \bm{x}),\\
    p(y = 0| \bm{x}) &= \theta(y = 0, e = 1| \bm{x}) + \theta(y = 0, e = 0| \bm{x})\\
    &= q(w = 0| \bm{x}) - \theta(y =1, e = 0| \bm{x})\\
    &= q(w = 0| \bm{x}) - p(y = 1| e = 0, \bm{x})\theta(e = 0| \bm{x}).
\end{align*}
Thus, we can identify $p(y| \bm{x})$ by using $q(w = 0| \bm{x})$, $p(y = 1| e = 0, \bm{x})$, and $\theta(e = 0| \bm{x})$. 

Combining this identification strategy with the strong ignorability assumption, we have
\begin{align}
    p(y = 1 | \bm{x}) = q(w = 1| \bm{x}) + p(y = 1|\bm{x})\theta(e = 0| \bm{x}) .
\end{align}
We exploit this equation to train our classifier. 

\subsection{ADPUE Learning in Population}
In this section, we present an automatic debiased learning method given access to population; that is, expected values of each observable random variable. 

\textbf{The population risk.}
We showed that under our identification strategy, we can identify the optimal classifier $p(y = 1| \bm{x})$. By using the identification strategy, we define our population risk as follows:
\begin{align*}
R(f)=&\mathbb{E}\big[\left(W_i + p(Y_i = 1| X_i)(1 - E_i)\right) \ell(f(X_i), 1)\\
&+ \left(1 - W_i - p(Y_i = 1| X_i)(1 - E_i)\right) \ell(f(X_i), 0)\big].
\end{align*}
Here, note that $R(f)$ is equal to $R^*(f)$ because
\begin{align*}
&\mathbb{E}\big[\left(W_i + p(Y_i = 1| X_i)(1 - E_i)\right) \ell(f(X_i), 1)\big]\\
&= \mathbb{E}\Big[\mathbb{E}\big[\left(q(W_i| X_i) + p(Y_i = 1| X_i)\theta(1-E_i | X_i)\right)| X_i\big]\ell(f(X_i), 1)\big]\\
&= \mathbb{E}\big[\mathbb{E}\Big[p(Y_i = 1| X_i)\ell(f(X_i), 1)\big],
\end{align*}
and similarly, $\mathbb{E}\big[\big(1-W_i - p(Y_i = 1| X_i)(1 - E_i)\big) \ell(f(X_i), 0)\big] = \mathbb{E}\Big[\mathbb{E}\big[p(Y_i = 0| X_i)\ell(f(X_i), 0)\big]$. 

\textbf{The pseudo classification risk.}
In our problem setting, $p(y = 1| \bm{x})$ is a target function to be learned from the training data; therefore, we cannot use the function to define the population risk. Instead of using the true $p(y = 1| \bm{x})$, we consider the following \emph{pseudo classification risk} by using some proxy function $f^\dagger\in\mathcal{F}$:
\begin{align}
\label{eq:basic}
\widetilde{R}&\left(f; f^\dagger\right)= \mathbb{E}\big[\left(W_i +  f^\dagger(X_i)(1 - E_i)\right) \ell(f(X_i), 1)\nonumber\\
&\ \ \ \ \ \ \ \  + \left(1-W_i -  f^\dagger(X_i)(1 - E_i)\right) \ell(f(X_i), 0)\big].
\end{align}
For the pseudo classification risk with the log loss function, let us denote a minimizer of \eqref{eq:basic} by $\widetilde{f}$, that is, $\widetilde{f} \in \argmin_{f\in\mathcal{F}}\widetilde{R}\left(f; f^\dagger\right)$. Then, we derive the following lemma. We present the proof in Appendix~\ref{appdx:lem:proof}. 
\begin{lemma}\label{lem:2}
For each $\bm{x}\in \mathcal{X}$, $\widetilde{f}(\bm{x}) = q(w=1|\bm{x}) + f^\dagger(\bm{x})\theta(e = 0|\bm{x})$.
\end{lemma}

\textbf{Alternate learning.}
Next, we consider the following alternate learning for $t=1,2,\dots,T$. 
Let $f^*_0$ be an initial guess of $p(y=1|\bm{x})$. Then, at round $t$, using 
$f^*_{t-1}$, we obtain 
\begin{align*}
    f^*_t = \argmin_{f\in\mathcal{F}} \widetilde{R}\left(f; f^*_{t-1}\right).
\end{align*}
Then, let $f^*_T$ be an output of the learning process. In the following theorem, we prove $f^*_T\to f^*$ as $T\to \infty$:
\begin{theorem}
For each $\bm{x}\in \mathcal{X}$, as $T\to \infty$, $f^*_T(\bm{x}) \to p(y = 1|\bm{x})$.
\end{theorem}
\begin{proof}
For each $\bm{x}\in\mathcal{X}$, from Lemma~\ref{lem:2}, $f^*_t = q(w=1|\bm{x}) + f^*_{t-1}(\bm{x})\theta(e = 0|\bm{x})$.
Therefore, we obtain 
\begin{align*}
    &f^*_T(\bm{x}) = q(w=1|\bm{x}) +  \left\{\left(q(w=1|\bm{x}) + f^*_{t-2}(\bm{x})\theta(e = 0|\bm{x})\right)\theta(e = 0|\bm{x})\right\}\\
    &= q(w=1|\bm{x})\sum^{T-1}_{t = 0} \theta^t(e = 0|\bm{x}) + f^*_{0}(\bm{x})\theta^T(e = 0|\bm{x}).
\end{align*}
Here, as $T\to \infty$,
\begin{align*}
    &q(w=1|\bm{x})\sum^{T-1}_{t = 0} \theta^t(e = 0|\bm{x})\to \frac{q(w=1|\bm{x})}{1 - \theta(e = 0|\bm{x})},
\end{align*}
and $f^*_{0}(\bm{x})\theta^T(e = 0|\bm{x}) \to 0$. Therefore, as $T\to \infty$,
\begin{align*}
    f^*_T \to \frac{q(w=1|\bm{x})}{\theta(e = 1|\bm{x})} = \frac{p(y = 1, e = 1|\bm{x})}{\theta(e = 1|\bm{x})} = p(y = 1|\bm{x}), 
\end{align*}
where we used $1 - \theta(e = 0|\bm{x}) = \theta(e = 1|\bm{x})$ and $p(y = 1, e = 1|\bm{x}) = p(y = 1|\bm{x})\theta(e = 1| \bm{x})$ under the strong ignorability assumption. 
The proof is complete. 
\end{proof}

\subsection{Proposed Method: ADPUE Learning}
Because expectations of random variables directly are not accessible, we replace them with their sample averages.

\textbf{Sample approximation.}
When we have training samples, we can naively replace the expectations by the corresponding sample averages. For a hypothesis set $\mathcal{H}$, which is a subset of a set of measurable functions, let us define the following empirical risk:
\begin{align*}
\widehat{R}\left(f; f^\dagger\right) =&\widehat{\mathbb{E}}^{\mathcal{D}^{\mathrm{PU}}}\big[W_i \ell(f(X_i), 1)+ \left(1 - W_i \right) \ell(f(X_i), 0)\big]\\
&+ \widehat{\mathbb{E}}^{\mathcal{D}^{\mathrm{E}}}\big[ f^\dagger(X_i)(1 - E_i)\ell(f(X_i), 1)\big]\\
&-  \widehat{\mathbb{E}}^{\mathcal{D}^{\mathrm{E}}}\big[f^\dagger(X_i)(1 - E_i)\ell(f(X_i), 0)\big],\nonumber
\end{align*}
where $\widehat{\mathbb{E}^{\mathrm{D}}}$ denotes an empirical mean using a dataset $\mathcal{D}$. Let $\widehat{f}_{n, 0}$ be an appropriately given initial guess of $f^*$. 

\textbf{Non-negative correction.}
In addition, to gain performance, we introduce a non-negative correction \citep{Kiryo2017} to our proposed empirical risk as
\begin{align*}
&\widehat{R}^{\mathrm{nnPUE}}\left(f; f^\dagger\right) =\\
&\widehat{\mathbb{E}}^{\mathcal{D}^{\mathrm{PU}}}\big[W_i \ell(f(X_i), 1)\big] + \widehat{\mathbb{E}}^{\mathcal{D}^{\mathrm{E}}}\big[ f^\dagger(X_i)(1 - E_i)\ell(f(X_i), 1)\big]\\
&\ + \max\Big\{\widehat{\mathbb{E}}^{\mathcal{D}^{\mathrm{PU}}}\big[\left(1 - W_i \right) \ell(f(X_i), 0)\big] -  \widehat{\mathbb{E}}^{\mathcal{D}^{\mathrm{E}}}\big[f^\dagger(X_i)(1 - E_i)\ell(f(X_i), 0)\Big], 0\Big\}.
\end{align*}
This correction is based on the fact that in the population,
\begin{align*}
    &\Big(q(w = 0) - p(y = 1| x)\theta(e = 0)\Big)\log(1 - f(\bm{x})) \leq 0\\
    &\implies \mathbb{E}\big[ \left(1 - W_i \right) \ell(f(X_i), 0) -  f^\dagger(X_i)(1 - E_i)\ell(f(X_i), 0)\big]\leq 0.
\end{align*}
The term $\widehat{\mathbb{E}}^{\mathcal{D}^{\mathrm{PU}}}\big[\left(1 - W_i \right) \ell(f(X_i), 0)\big] -  \widehat{\mathbb{E}}^{\mathcal{D}^{\mathrm{E}}}\big[f^\dagger(X_i)(1 - E_i)\ell(f(X_i), 0)\big]$ can be negative because $\mathcal{D}^{\mathrm{PU}}$ and $\mathcal{D}^{\mathrm{E}}$ are separate observations. 

\textbf{Empirical risk minimization.}
Then, for each $t = 1,2,\dots, T$, we define an empirical minimizer as 
\begin{align*}
     \hat{f}_{n, t} = \argmin_{f\in\mathcal{H}}\widehat{R}^{\mathrm{nnPUE}}\left(f; \hat{f}_{n, t-1}\right).
\end{align*}
Let $\hat{f}_{n, T}$ be the final output of our proposed algorithm. We show the pseudo code in Algorithm~\ref{alg}. 

In addition, we can directly minimize the empirical risk as
\begin{align*}
     &\hat{f}_{n} = \argmin_{f\in\mathcal{H}}\widehat{R}^{\mathrm{nnPUE}}\left(f\right),
\end{align*}
where $\widehat{R}^{\mathrm{nnPUE}}\left(f\right) := \widehat{R}^{\mathrm{nnPUE}}\left(f; f\right)$. Although we do not theoretically justify minimizing the empirical risk, we experimentally confirm that this formulation also works well. Because direct minimization is easier than the alternate learning, we use this formulation in our experiments although the alternate learning also works.  

We call our proposed method automatic debiased PUE (ADPUE), which trains a classifier by minimising the empirical risks, because unlike related methods, such as the propensity score in EM method of \citet{Bekker2018}, our proposed method can directly debias the selection bias without requiring intermediate estimates.

\begin{algorithm}[tb]
   \caption{ADPUE learning}
   \label{alg}
\begin{algorithmic}
   \STATE {\bfseries Input:} Training data $\mathcal{D}^{\mathrm{PU}}$ and $\mathcal{D}^{\mathrm{E}}$. 
   \STATE {\bfseries Initialization:} Set $\hat{f}_{n, 0}$. 
   \FOR{$t=1$ to $T$}
   \STATE Obtain $\hat{f}_{n, t} = \argmin_{f\in\mathcal{H}}\widehat{R}\left(f; \hat{f}_{n, t-1}\right)$. 
   \ENDFOR
   \STATE Return $\hat{f}_{n, T}$. 
\end{algorithmic}
\vspace{-0.1cm}
\end{algorithm} 

\section{Variants of PUE Classification}
\label{sec:other_form}
In addition to the DGP in Section~\ref{sec:obs}, we can consider various formulations of the problem of PUE classification using our basic formulation and idea. Here, for example, we introduce two problem instances with different DGPs. These settings are related to two-sample scenario, rather than one-sample scenario in Section~\ref{sec:obs}. 

\subsection{PE Classification}
First, we introduce a situation where we can only observe positive data and exposure data as
\begin{align*}
    \mathcal{D}^{\mathrm{P}} &= \left\{\big( X_i\big)\right\}^{n^{\mathrm{P}}}_{i=1}\iid \zeta(\bm{x}| w = 1),\\
    \mathcal{D}^{\mathrm{E}} &= \left\{\big(E_i, X_i\big)\right\}^{n^{\mathrm{E}}}_{i=1} \iid \theta(e|\bm{x})\zeta(\bm{x}),
\end{align*}
We also assume that the class prior $q(w = 1)$ is known. 
We call this problem classification from positive and exposure data (PE classification).

If we estimate $p(e=1|x)$ and $p(y=1, e=1|x)$ separately, we can still estimate $p(y=1|x)$. However, separate estimation is known as costly because it requires learning two models separately. Moreover, as our method, it is often empirically reported that an end-to-end method can improve performances. Therefore, we consider direct empirical risk minimization.

For this problem, we define our empirical risk as
\begin{align*}
&\widehat{R}^{\mathrm{nnPE}}\left(f; f^\dagger\right) = q(w = 1) \widehat{\mathbb{E}}^{\mathcal{D}^{\mathrm{P}}}\big[\log (f(X_i))\big]\\
&\ + \max\big\{\widehat{\mathbb{E}}^{\mathcal{D}^{\mathrm{E}}}\big[ \ell(f(X_i), 0)\big]  - q(w = 1) \widehat{\mathbb{E}}^{\mathcal{D}^{\mathrm{P}}} \big[\log (1 - f(X_i))\big], 0\big\}\\
&\ + \widehat{\mathbb{E}}^{\mathcal{D}^{\mathrm{E}}}\big[ f^\dagger(X_i)(1 - E_i)\ell(f(X_i), 1)\big]\\
&\ -  \widehat{\mathbb{E}}^{\mathcal{D}^{\mathrm{E}}}\big[f^\dagger(X_i)(1 - E_i)\ell(f(X_i), 0)\big].
\end{align*}
We call a learning method minimizing this empirical risk Automatic Debiased PE (ADPE) learning.

\subsection{Fully Separated PUE Classification}
Next, we consider a situation where we can observe positive, unlabeled, and exposure data, separately; that is, we observe
\begin{align*}
    \mathcal{D}^{\mathrm{P}} &= \left\{\big( X_i\big)\right\}^{n^{\mathrm{P}}}_{i=1}\iid \zeta(\bm{x}| w = 1),\\
    \mathcal{D}^{\mathrm{U}} &= \left\{\big( X_i\big)\right\}^{n^{\mathrm{U}}}_{i=1}\iid \zeta(\bm{x}),\\
    \mathcal{D}^{\mathrm{E}} &= \left\{\big(E_i, X_i\big)\right\}^{n^{\mathrm{E}}}_{i=1} \iid \theta(e|\bm{x})\zeta(\bm{x}),
\end{align*}
We also assume that the class prior $q(w = 1)$ is known. We call this problem the fully separated PUE classification (FPUE classification).

We replace $\widehat{\mathbb{E}}^{\mathcal{D}^{\mathrm{E}}}\big[\left(1 - Y_i \right) \ell(f(X_i), 0)\big]$ in $\widehat{R}^{\mathrm{nnPE}}\left(f; f^\dagger\right)$ with $\widehat{\mathbb{E}}^{\mathcal{D}^{\mathrm{U}}\cap \mathcal{D}^{\mathrm{E}}}\big[\left(1 - Y_i \right) \ell(f(X_i), 0)\big]$, where $\widehat{\mathbb{E}}^{\mathcal{D}^{\mathrm{U}}\cap \mathcal{D}^{\mathrm{E}}}[\cdot]$ denotes the sample average over a joint set of $\mathcal{D}^{\mathrm{U}}$ and $\mathcal{D}^{\mathrm{E}}$. We call a learning method minimizing this empirical risk the Automatic Debiased FPUE (ADFPU) learning.

\section{Extensions to Semi-Supervised Classification}
In this section, we extend our method for the PUE classification to a setting of semi-supervised classifications, where we can use labeled positive and negative data and unlabeled data. We present two different settings, called Semi-Supervised Classification with Exposure Labels (SSE) and Separated SSE (3SE), respectively. 

\subsection{SSE Classification}
We consider a situation where exposure labels are simultaneously observed with positive and unlabeled data in the PUE classification problem, called the SSE classification problem.
In the standard semi-supervised classification, we assume that whether a sample is labeled or unlabeled is determined deterministically. However, in many real-world applications, a sample is often randomly selected as labeled or unlabeled data, and the selection is correlated with the (potential) label. Our formulation is one of the problem settings. We illustrate the concept in Figure~\ref{fig:difference}.

\begin{figure}[t]\centering
\includegraphics[scale=0.4]{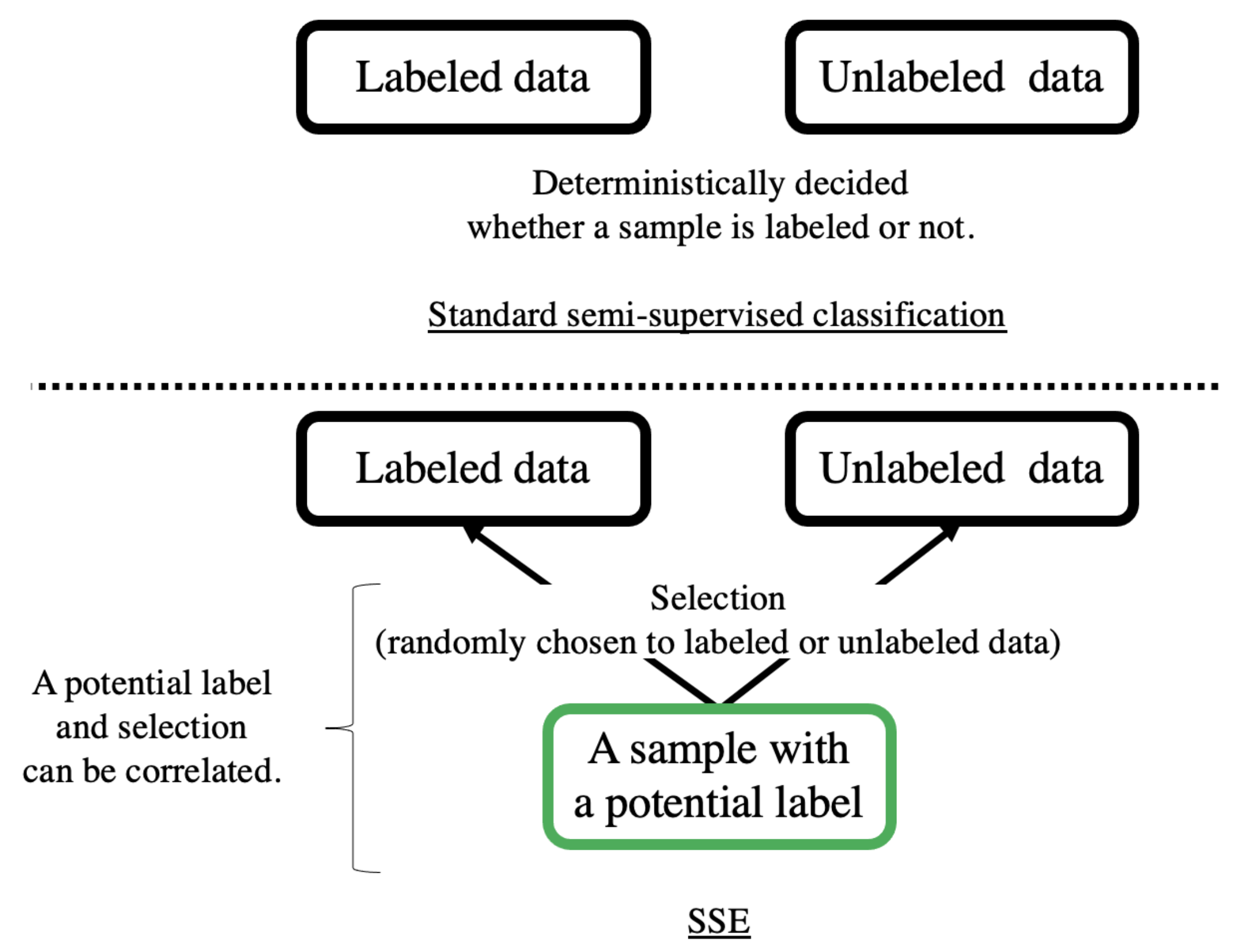}
\vspace{-0.3cm}
\caption{Difference between the standard semi-supervised classification and SSE.}
\vspace{-0.3cm}
\label{fig:difference}
\end{figure}

\textbf{Observations.}
In the SSE problem, we observe
\begin{align}
    &\mathcal{D}^{\mathrm{SSE}} = \left\{\Big(W_i, E_i, X_i\Big)\right\}^{n^{\mathrm{SSE}}}_{i=1},
\end{align}
where $W_i = E_iY_i$, $q\big(w|e, \bm{x}\big)$ is the conditional density of $W_i$ given $(E_i, X_i)$, and $\theta\big(e|\bm{x}\big)$ is that of $E_i$ given $X_i$. We observe a set $\mathcal{D} = \big\{\big(W_i, E_i, X_i\big)\big\}^n_{i=1}$ and call it training data. 

\textbf{Automatic Debiased Supervised (ADS) learning.}
First, we demonstrate that in this setting, we can train $p(y = 1|\bm{x})$ by applying the standard supervised classification method directly. From $q(w|e, \bm{x}) = \frac{\theta(w, e| \bm{x})}{\theta(e| \bm{x})}$, under the strong ignorability assumption, we have
\begin{align*}
    &q(w|e = 1, \bm{x}) = \frac{\theta(w, e = 1| \bm{x})}{\theta(e = 1| \bm{x})} = \frac{\theta(y, e = 1| \bm{x})}{\theta(e = 1| \bm{x})}= p(y|\bm{x}),\\
    &q(w=1|e = 0, \bm{x}) = \frac{\theta(w=1, e = 0| \bm{x})}{\theta(e = 0| \bm{x})} = 0,\\
    &q(w=0|e = 0, \bm{x}) = \frac{\theta(w=0, e = 0| \bm{x})}{\theta(e = 0| \bm{x})} = 1.
\end{align*}
Let $\mathcal{D}^{\mathrm{L}} = \{(W_i, X_i)\}^{n^{\mathrm{L}}}$ be a subset of $\mathcal{D}^{\mathrm{SSE}}$, which are samples with $E_i = 1$. Note that $W_i = Y_i E_i =  Y_i$ in the subset. Let us consider the following empirical risk:
\begin{align*}
    &\widehat{R}^{\mathrm{L}}\left(f\right) =\widehat{\mathbb{E}}^{\mathcal{D}^{\mathrm{L}}}\Big[W_i\ell(f(X_i), 1) + (1 - W_i)\ell(f(X_i), 0)\Big].
\end{align*}

This empirical risk is unbiased with regard to $R^*(f)$ because
\begin{align*}
    &\mathbb{E}\left[\widehat{R}^{\mathrm{L}}\left(f\right) \right]= \mathbb{E}\left[W_i\ell(f(X_i), 1) + (1 - W_i)\ell(f(X_i), 0)\right]\\
    &= \theta(e = 1)\mathbb{E}\big[Y_i\ell(f(X_i), 1) + (1 - Y_i)\ell(f(X_i), 0)| E_i = 1\big]. 
\end{align*}
Here, we have $\mathbb{E}\big[Y_i\ell(f(X_i), 1) + (1 - Y_i)\ell(f(X_i), 0)| E_i = 1\big] = \mathbb{E}\big[p(Y_i=1| X_i)\ell(f(X_i), 1) + p(Y_i=0| X_i)\ell(f(X_i), 0)|E_i = 1\big]$.
Then, we have
\begin{align*}
    &\mathbb{E}\left[\widehat{R}^{\mathrm{L}}\left(f\right) \right]= \theta(e = 1)\mathbb{E}\big[p(Y_i=1| X_i)\ell(f(X_i), 1) + p(Y_i=0| X_i)\ell(f(X_i), 0)|E_i = 1\big]\\
    &= \mathbb{E}\Big[p(Y_i=1| X_i)\ell(f(X_i), 1)+ p(Y_i=0| X_i)\ell(f(X_i), 0)\Big],
\end{align*}
where for a function $Q(y, x)$, we used $\theta(e = 1)\int Q(y, x)p(y, \bm{x}|e = 1)\mathrm{d}y\mathrm{d}\bm{x} = \int Q(y, x)p(y, \bm{x}, e = 1)\mathrm{d}y\mathrm{d}\bm{x} = \mathbb{E}\left[Q(Y_i, X_i)\right]$. 

Therefore, $f = \argmax_{f\in\mathcal{H}}\widehat{R}^{\mathrm{L}} \left(f\right)$ is a classifier obtained by minimization of an unbiased empirical risk. Thus, under the strong ignorability assumption, we can perform unbiased risk minimization only by using labeled samples. However, this approach does not employ a whole unlabeled samples. In this sense, this algorithm is a supervised learning, not semi-supervised learning.  Because many existing studies report that semi-supervised learning, which utilizes unlabeled data in the training process, can improve empirical performance of an algorithm, we consider an algorithm using both labeled and unlabeled data. We call a method minimizing this empirical risk the ADS learning. 

\textbf{Automatic Debiased Semi-Supervised (ADSS) learning.}
As well as the ADPUE learning, we can construct the following empirical risk: $\widehat{R}^{\mathrm{SS}}\left(f; f^\dagger\right) =$
\begin{align*}
&\widehat{\mathbb{E}}^{\mathcal{D}^{\mathrm{PUE}}}\Big[W_i \ell(f(X_i), 1) + f^\dagger(X_i)(1 - E_i)\ell(f(X_i), 1)\\
&\ +\left(1 - W_i \right) \ell(f(X_i), 0) -  f^\dagger(X_i)(1 - E_i)\ell(f(X_i), 0)\Big].
\end{align*}
We do not impose a non-negative correction because the constraint is not violated if $f(\bm{x})\in (0, 1)$. We call a method minimizing this empirical risk the ADSS learning. 

\textbf{Double ADSS learning.}
We can minimize the empirical risk of supervised learning with that of semi-supervised learning simultaneously. 
\begin{align*}
    \widehat{R}^{\mathrm{DoubleADSS}}\left(f; f^\dagger\right) = \alpha\widehat{R}^{\mathrm{nnSS}}\left(f; f^\dagger\right)  + (1-\alpha)\widehat{R}^{\mathrm{L}}\left(f; f^\dagger\right),
\end{align*}
where $\alpha \in [0, 1]$ is a weight. We call a method minimizing this empirical risk the Double ADSS (DADSS) learning. 

\subsection{3SE Classification}
In the 3SE problem. we consider the following DGP:
\begin{align*}
    \mathcal{D}^{\mathrm{SSE}} = \left\{\big(W_i, E_i, X_i\big)\right\}^{n^{\mathrm{SSE}}}_{i=1},\quad \mathcal{D}^{\mathrm{PU}} = \left\{\big(W_i,  X_i\big)\right\}^{n^{\mathrm{PU}}}_{i=1},
\end{align*}
In this case, in addition to $\widehat{R}^{\mathrm{L}}\left(f; f^\dagger\right)$, we have the following empirical risk:
\begin{align*}
\widehat{R}^{\mathrm{nnPUE}}\left(f; f^\dagger\right) &=\widehat{\mathbb{E}}^{\mathcal{D}^{\mathrm{PUE}}\cap \mathcal{D}^{\mathrm{PU}}}\Big[W_i \ell(f(X_i), 1)\Big] + \widehat{\mathbb{E}}^{\mathcal{D}^{\mathrm{PUE}}}\Big[ f^\dagger(X_i)(1 - E_i)\ell(f(X_i), 1)\Big]\\
&\ \ \ + \max\Big\{\widehat{\mathbb{E}}^{\mathcal{D}^{\mathrm{PUE}}\cap \mathcal{D}^{\mathrm{PU}}}\Big[\left(1 - W_i \right) \ell(f(X_i), 0)\Big] -  \widehat{\mathbb{E}}^{\mathcal{D}^{\mathrm{PUE}}}\Big[f^\dagger(X_i)(1 - E_i)\ell(f(X_i), 0)\Big], 0\Big\}.
\end{align*}
We call a method minimizing this empirical risk the Double Automatic Debiased 3SE (AD3SE) learning, which is effective under a situation where the sample size of $\mathcal{D}^{\mathrm{PUE}}$ is small, but that of $\mathcal{D}^{\mathrm{PU}}$ is large.

\section{Related Work}
\label{sec:related}

\begin{table}[t]
 \caption{Comparison of the problem settings.}
 \vspace{-0.2cm}
    \label{table:comparison}
    \centering
    \scalebox{0.8}[0.8]{
    \begin{tabular}{|l|c|c|c|c|}
    \hline
        & P  & U & E & \\
    \hline
       PUE & \multicolumn{2}{|c|}{$\circ$ (one-sample)} & $\circ$ & close to one-sample PU \\
    \hline
       PE & $\circ$ & \multicolumn{2}{|c|}{$\circ$ (one-sample)} & close to two-sample PU\\
    \hline
       FSPUE & $\circ$ & $\circ$ & $\circ$ & close to two-sample PU \\
    \hline
       SSE & \multicolumn{3}{|c|}{$\circ$ (one-sample)} & close to one-sample PU/PNU\\
    \hline
       3SE & \multicolumn{3}{|c|}{$\circ$ (one-sample + PU data) } & close to one-sample PU/PNU\\
     \hline
       one-sample PU & \multicolumn{2}{|c|}{$\circ$ (one-sample)} & - &  \\
     \hline
       two-sample PU & $\circ$ & $\circ$  & - &   \\
    \hline
    \end{tabular}
    }
    \vspace{-0.3cm}
\end{table}

% table 4 place

The one sample and two sample settings are also called the
censoring scenario and case-control scenario, respectively \citep{elkan2008learning}. 
As noted by \citet{Gang2016}, the former is slightly more general than the latter. For the two-sample setting, \citet{duPlessis2015} suggests the use of an \emph{unbiased estimator} of the classification risk, known as \emph{unbiased PU learning}. \citet{Kiryo2017} proposes a non-negative correction to prevent the overfitting of neural networks, which has been applied to other problem settings \citep{Nan2020,Kato2021bregman}. The PU classification methods have been applied to the semi-supervised classification problem \citep{Sakai2017,Sakai2018}. 

As explained by \citet{elkan2008learning}, identifying a classifier without assuming how the positive data is labeled is impossible. Thus, the SCAR assumption is typically employed, which posits that the positive labeled data has the same distribution as the positive unlabeled data. 
However, the SCAR assumption is often unrealistic in many PU learning scenarios, such as a patient's electronic health record \citep{Bekker2018} and a recommendation system \citep{Benjamin2009,pmlr-v48-schnabel16}.
In these cases, selection bias \citep{angrist_mostly_2008} may exist, whereby the distribution of the positive data differs between the labeled and unlabeled data. To address this, alternative assumptions have been proposed \citep{Bekker2018,Bekker2019,kato2018pubp,Hsieh2019,Zhao2021}. 

One of the important applications of PUE classification is recommender systems \citep{Yifan2008,Jannach2018,Liang2016}, where we often employ implicit feedback data \cite{Wang2018}, which is similar to PUE classification. \cite{Yifan2008} proposes a weighted matrix factorization, assigning less weight to unclicked items to account for the lower confidence in predictions compared to clicked items. \cite{Liang2016} proposes exposure matrix factorization based on the latent probabilistic model. \cite{Saito2020} applies the IPW method to this problem.

\begin{table*}[ht]
\begin{center}
\vspace{-0.3cm}
\caption{Experiments using semi-synthetic datasets. The best-performing method in each setting is highlighted in \textbf{bold}.}
\vspace{-0.2cm}
\label{tbl:table1}
\scalebox{0.77}[0.77]{
\begin{tabular}{|l||rrrr|rrr||rrrr|rrr|}
\hline
test data &     \multicolumn{7}{|c||}{Inductive}  &      \multicolumn{7}{|c|}{Transductive} \\
\hline
problem setting &      \multicolumn{4}{|c|}{PUE (close to one-sample PU)}   &     \multicolumn{3}{|c||}{3SE}  &     \multicolumn{4}{|c|}{PUE (close to two-sample PU)}   &     \multicolumn{3}{|c|}{3SE} \\
\hline
{} &     Logit  &     uPU  &     EM  &     ADPUE  &     ADS  &     PNU  &     AD3SE  &     Logit  &     uPU  &     EM  &     ADPUE  &     ADS  &     PNU  &     AD3SE \\
\hline
& \multicolumn{14}{|c|}{sample splitting ratio $\alpha = 0.3$} \\
\hline
{\tt australian} &  0.661 &  0.606 &  0.591 &  \textbf{ 0.742} &  0.798 &  0.689 &  \textbf{ 0.849} &  0.626 &  0.606 &  0.591 &  \textbf{ 0.723} &  0.791 &  0.644 &  \textbf{ 0.839} \\
{\tt w8a} &  0.616 &  0.600 &  0.539 &  \textbf{ 0.704} &  \textbf{ 0.781} &  0.753 &  0.767 &  0.593 &  0.579 &  0.539 &  \textbf{ 0.689} &  \textbf{ 0.763} &  0.728 &  0.752 \\
{\tt covtype} &  0.562 &  0.558 &  \textbf{0.636} &  0.580 &  0.628 &  \textbf{0.645} &  0.616 &  0.569 &  0.566 &  \textbf{ 0.636} &  0.582 &  0.637 &  \textbf{ 0.650} &  0.617 \\
{\tt mushrooms} &  0.193 &  0.150 &  0.645 &  \textbf{ 0.884} &  \textbf{ 1.000} &  0.880 &  0.999 &  0.191 &  0.149 &  0.645 &  \textbf{ 0.882} &  \textbf{ 1.000} &  0.860 &  0.999 \\
{\tt german} &  \textbf{ 0.706} &  0.698 &  0.648 &  0.601 &  0.712 &  \textbf{ 0.723} &  0.715 &  \textbf{ 0.706} &  0.698 &  0.648 &  0.603 &  0.708 &  \textbf{ 0.718} &  0.711 \\
\hline
\hline
& \multicolumn{14}{|c|}{sample splitting ratio $\alpha = 0.5$} \\
\hline
{\tt australian} &  0.667 &  0.606 &  0.660 &  \textbf{ 0.680} &  0.784 &  0.649 &  \textbf{ 0.801} &  0.641 &  0.621 &  \textbf{ 0.660} &  0.659 &  0.781 &  0.637 &  \textbf{ 0.792} \\
{\tt w8a} &  0.663 &  0.659 &  0.543 &  \textbf{ 0.726} &  0.765 &  0.748 &  \textbf{ 0.766} &  0.631 &  0.628 &  0.543 &  \textbf{ 0.705} &  0.743 &  0.713 &  \textbf{ 0.746} \\
{\tt covtype} &  0.561 &  0.558 &  \textbf{ 0.651} &  0.592 &  0.639 &  \textbf{0.644} &  0.626 &  0.567 &  0.567 &  \textbf{ 0.651} &  0.589 &  0.644 &  \textbf{ 0.646} &  0.629 \\
{\tt mushrooms} &  0.160 &  0.145 &  0.680 &  \textbf{ 0.876} &  \textbf{ 1.000} &  0.834 &  0.998 &  0.154 &  0.142 &  0.680 &  \textbf{ 0.881} &  \textbf{ 1.000} &  0.814 &  0.998 \\
{\tt german} &  \textbf{ 0.707} &  0.700 &  0.669 &  0.583 &  0.659 &  \textbf{ 0.718} &  0.693 &  \textbf{ 0.705} &  0.696 &  0.669 &  0.580 &  0.651 &  \textbf{ 0.712} &  0.688 \\
\hline
\end{tabular}
}
\end{center}
\vspace{-0.2cm}

\begin{center}
\caption{Experiments using neural networks. The best-performing method in each setting is highlighted in \textbf{bold}.}
\vspace{-0.2cm}
\label{tbl:exp:dataset_neural}
\scalebox{0.79}[0.78]{
\begin{tabular}{|l|c||rrr|rrr||rrr|rrr|}
\hline
\multicolumn{2}{|l||}{test data}& \multicolumn{6}{c||}{Inductive}& \multicolumn{6}{c|}{Transductive} \\
\hline
\multicolumn{2}{|l||}{problem setting} &      \multicolumn{3}{|c|}{PUE}   &     \multicolumn{3}{|c||}{3SE}  &     \multicolumn{3}{|c|}{PUE}   &     \multicolumn{3}{|c|}{3SE} \\
\hline
\multicolumn{1}{|l}{dataset}     & \multicolumn{1}{|c||}{separate\_ratio} & \multicolumn{1}{c}{Logit} & \multicolumn{1}{c}{nnPU} & \multicolumn{1}{c|}{ADPUE} & \multicolumn{1}{c}{ADS} & \multicolumn{1}{c}{PNU} & \multicolumn{1}{c||}{AD3SE} & \multicolumn{1}{c}{Logit} & \multicolumn{1}{c}{nnPU} & \multicolumn{1}{c|}{ADPUE} & \multicolumn{1}{c}{ADS} & \multicolumn{1}{c}{PNU} & \multicolumn{1}{c|}{AD3SE} \\
\hline
\multirow{5}{*}{\texttt{MNIST}} & 0.1 & 0.737 &  0.758 &  \textbf{0.893} &  \textbf{0.977} &  0.779 &  0.974 &                   0.703 &  0.735 &  \textbf{0.865} &  \textbf{0.973} &  0.653 &  0.964 \\
              & 0.3 &      0.750 &  0.838 &  \textbf{0.923} &  \textbf{0.976} &  0.768 &  0.966 &                   0.711 &  0.807 &  \textbf{0.898} &  \textbf{0.970} &  0.643 &  0.946 \\
              & 0.5 &      0.757 &  0.835 &  \textbf{0.944} &  \textbf{0.968} &  0.738 &  0.965 &                   0.710 &  0.807 &  \textbf{0.923} &  \textbf{0.965} &  0.631 &  0.942 \\
              & 0.7 &      0.768 &  0.842 &  \textbf{0.943} &  0.959 &  0.698 &  \textbf{0.960} &                   0.715 &  0.805 &  \textbf{0.921} &  \textbf{0.955} &  0.605 &  0.937 \\
              & 0.9 &      0.754 &  0.869 &  \textbf{0.945} &  0.937 &  0.600 &  \textbf{0.953} &                   0.713 &  0.839 &  \textbf{0.923} &  \textbf{0.931} &  0.558 &  0.929 \\
\hline
\multirow{5}{*}{\texttt{Fashion-MNIST}} & 0.1 &      0.847 &  \textbf{0.917} &  0.885 &  \textbf{0.972} &  0.882 &  0.970 &                   0.816 &  \textbf{0.905} &  0.829 &  \textbf{0.971} &  0.822 &  0.965 \\
              & 0.3 &      0.838 &  \textbf{0.931} &  0.928 &  \textbf{0.970} &  0.864 &  0.966 &                   0.803 &  \textbf{0.922} &  0.900 &  \textbf{0.970} &  0.804 &  0.954 \\
              & 0.5 &      0.840 &  0.939 &  \textbf{0.950} &  \textbf{0.970} &  0.858 &  0.962 &                   0.809 &  0.929 &  \textbf{0.933} &  \textbf{0.966} &  0.801 &  0.944 \\
              & 0.7 &      0.858 &  0.922 &  \textbf{0.953} &  \textbf{0.965} &  0.838 &  0.963 &                   0.817 &  0.906 &  \textbf{0.935} &  \textbf{0.967} &  0.787 &  0.942 \\
              & 0.9 &      0.845 &  0.913 &  \textbf{0.949} &  \textbf{0.961} &  0.777 &  0.952 &                   0.802 &  0.889 &  \textbf{0.930} &  \textbf{0.958} &  0.755 &  0.930 \\
\hline
\multirow{5}{*}{\texttt{CIFAR-10}} & 0.1 & 0.752 &  \textbf{0.838} &  0.809 &  \textbf{0.902} &  0.827 &  0.896 &                   0.723 &  \textbf{0.827} &  0.806 &  \textbf{0.903} &  0.773 &  0.895 \\
        & 0.3 &      0.752 &  0.820 &  \textbf{0.825} &  \textbf{0.897} &  0.808 &  0.892 &                   0.725 &  0.797 &  \textbf{0.819} &  \textbf{0.899} &  0.761 &  0.886 \\
        & 0.5 &      0.770 &  0.815 &  \textbf{0.822} &  \textbf{0.892} &  0.796 &  0.888 &                   0.733 &  0.786 &  \textbf{0.810} &  \textbf{0.894} &  0.745 &  0.879 \\
        & 0.7 &      0.746 &  0.825 &  \textbf{0.848} &  \textbf{0.886} &  0.764 &  0.873 &                   0.719 &  0.795 &  \textbf{0.838} &  \textbf{0.884} &  0.728 &  0.866 \\
        & 0.9 &      0.759 &  0.841 &  \textbf{0.861} &  0.864 &  0.406 &  \textbf{0.880} &                   0.715 &  0.805 &  \textbf{0.855} &  0.854 &  0.388 &  \textbf{0.873} \\

\hline
\end{tabular}
}
\end{center}
\vspace{-0.5cm}
\end{table*}

\section{Experiments}
In this section, we report experimental results which were conducted using semi-synthetic datasets. We investigate the performances under the PUE and 3SE classification problems, called PUE and 3SE, respectively. Note that the 3SE includes the SSE classification problem as a special case.

\paragraph{Experiments with linear models.} First, we conduct experiments using simple linear models for all methods
We use five classification datasets, {\tt australian}, {\tt w8a}, {\tt covtype}, {\tt mushrooms}, and 
{\tt german}, from LIBSVM\footnote{The data is available from  \url{https://www.csie.ntu.edu.tw/~cjlin/libsvmtools/}.}. The information of datasets are summarized in Appendix~\ref{appdx:det_exp} Table~\ref{tbl:exp:dataset}. To make the experimental condition the same, for {\tt w8a}, {\tt covtype} and {\tt mushrooms}, we only use $1,800$ samples. 

For exposure labels, we use the following conditional probability: for $\bm{x} = (x_1, x_2,\dots, x_d)\in\mathcal{X}\subset \mathbb{R}^d$, 
\[
    \theta(e=1|\bm{x}) = \frac{C}{1 + \exp\big(-\big\{x_{13}g_1(\bm{x}) + (1 - x_{13})g_2(\bm{x})\big\}\big)},
\]
where $g_1(\bm{x}) =  x_2 + 2x_3 + 3x_4x_5 + 4x_6 + 5x^2_6$, $g_2(\bm{x}) = x_7 + 2x_8 + 3x_9x_{10} + 4x_{11} + 5x^2_{12}$, and $C$ is a constant. We adjust the constant $C$ to set the unconditional probability $\theta(e=1) = \int \theta(e=1|\bm{x})\zeta(\bm{x})\mathrm{d}\bm{x}$ at a certain value.  

We investigate the performances of our ADPUE, ADS, and AD3SE learning methods. Note that the ADS is a learning method for the SSE setting. Because the 3SE includes the SSE, we can use the ADS in the 3SE. We call standard logistic regression using all observed positive data ``Logit,'' EM method proposed by \citet{Bekker2018} ``EM,''
unbiased PU learning proposed by \citet{duPlessis2015} ``uPU'', PNE learning proposed by \citet{Sakai2017} ``PNU.'' Note that logistic regression only using observed and exposed positive data is the ADS learning. 

The setting of PUE is close to the two-sample PU learning, and we observe $\mathcal{D}^{\mathrm{PU}} = \left\{\big(W_i, X_i\big)\right\}^{n^{\mathrm{P}}}_{i=1}$ and $\mathcal{D}^{\mathrm{E}} = \left\{\big(E_i, X_i\big)\right\}^{n^{\mathrm{E}}}_{i=1}$. The setting of 3SE is a special case of the semi-supervised classification problem, and we observe $\mathcal{D}^{\mathrm{SSE}} = \left\{\big(W_i, E_i, X_i\big)\right\}^{n^{\mathrm{SSE}}}_{i=1}$ and $\mathcal{D}^{\mathrm{PU}} = \left\{\big(W_i,  X_i\big)\right\}^{n^{\mathrm{PU}}}_{i=1}$. The difference from the setting of PUE is the existence of $W_i$ in $\mathcal{D}^{\mathrm{SSE}}$. Given the whole $n$ training samples, we split them into $\mathcal{D}^{\mathrm{PU}}$ and $\mathcal{D}^{\mathrm{E}}$, or $\mathcal{D}^{\mathrm{PU}}$ and $\mathcal{D}^{\mathrm{SSE}}$ with the ratio $\alpha:(1 - \alpha)$, where $\alpha$ is a sample splitting ratio. We report experimental results for $\alpha=0.3$ and $\alpha=0.5$.

For a prediction model, we use simple linear models; that is, $\sum^d_{k=1}\gamma_kx_{k}$ for some parameter $\gamma_k$. Among the samples, we use randomly chosen $300$ samples as test data. We conduct $100$ trials for each setting and show the accuracy of the classifier by computing the average of the correct answer ratios (accuracy) in prediction for test data (inductive) and unlabeled data in the train data (transductive). 
We jointly show the results in Table~\ref{tbl:table1}. For the PUE, the proposed method outperforms other methods. This is due to the limited number of methods that can solve this problem. For the 3SE, other methods, such as the ADS, also perform well.

\paragraph{Experiments with neural networks.} Next, we conduct experiments using neural networks instead of simple linear regression models. Our experiments use two binary classification datasets, which are constructed from the \texttt{MNIST}, \texttt{Fashion-MNIST}, and \texttt{CIFAR-10} datasets. The details of the datasets are described in Appendix~\ref{appdx:det_exp} with additional results. There are $70,000$ samples in \texttt{MNIST}, \texttt{Fashion-MNIST} datasets and $60,000$ samples in \texttt{CIFAR-10} dataset.
We employ the same procedure as the previous experiments using linear  models. Among the samples, we use randomly choose $10,000$ samples as test data. We conduct $100$ trials for each setting and show the accuracy of the classifier by computing the average of the correct answer ratios (accuracy) in prediction for test data (inductive) and unlabeled data in the train data (transductive). 
We jointly show the numerical results in Table~\ref{tbl:exp:dataset_neural}, as well as the accuracy score in Figure~\ref{fig:accuracy}, the precision and recall in Figure~\ref{fig:precision_and_recall}. Obviously, our proposed methods show higher accuracy than other methods, which indicates their excellent performance.

\begin{figure}[h]
\centering
\subfigure{\includegraphics[width=0.3\linewidth, height=0.23\textwidth]{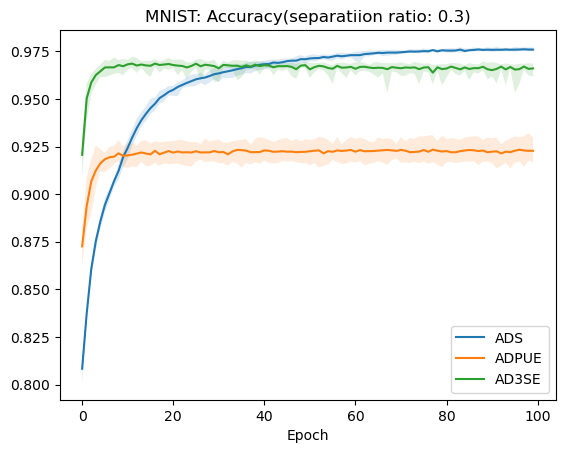}}
\subfigure{\includegraphics[width=0.3\linewidth,height=0.23\textwidth]{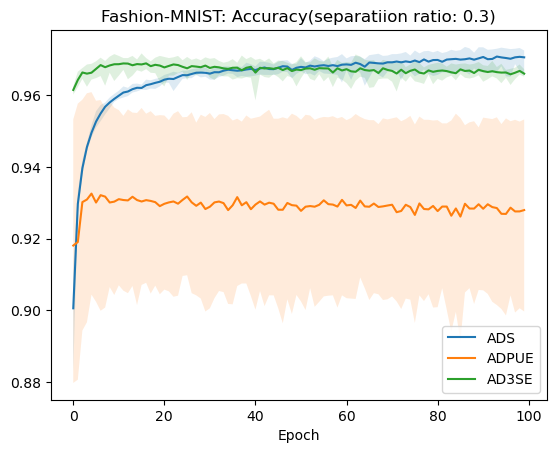}}
\subfigure{\includegraphics[width=0.3\linewidth,height=0.23\textwidth]{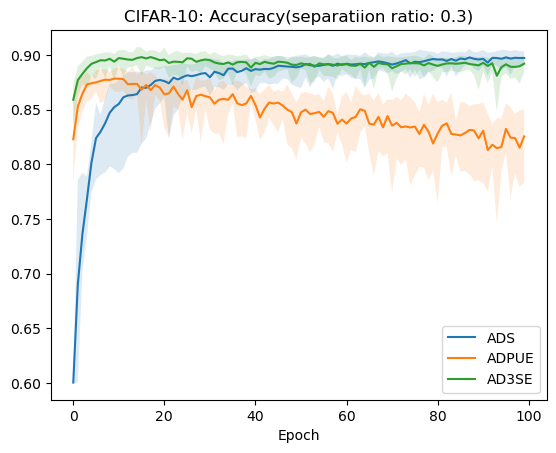}}
\subfigure{\includegraphics[width=0.3\linewidth,height=0.23\textwidth]{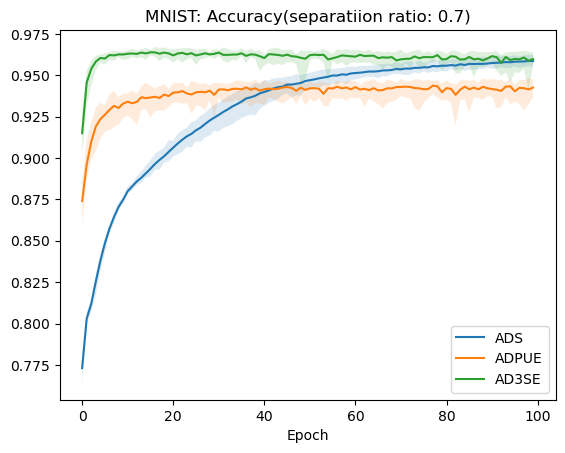}}
\subfigure{\includegraphics[width=0.3\linewidth,height=0.23\textwidth]{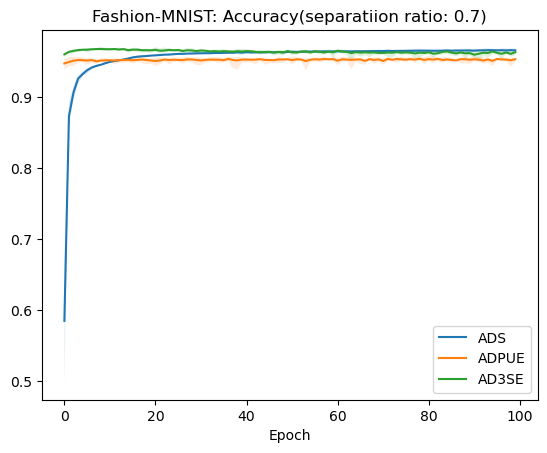}}
\subfigure{\includegraphics[width=0.3\linewidth,height=0.23\textwidth]{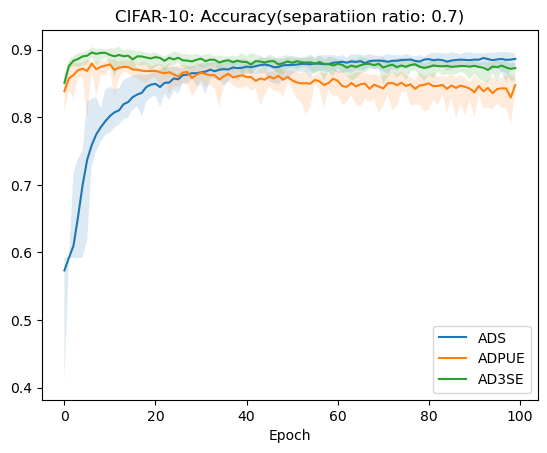}}
\vspace{-0.2cm}
\caption{Results of neural networks experiments. These plots show the accuracy score of training epochs for two different separation ratio and three datasets. (Upper: separation ratio = $0.3$, Lower: separation ratio = $0.7$. Left: \texttt{MNIST}, Center: \texttt{Fashion-MNIST}, Right: \texttt{CIFAR-10})}
\label{fig:accuracy}
\end{figure}

\begin{figure}[h]
\centering
\subfigure{\includegraphics[width=0.3\linewidth, height=0.23\textwidth]{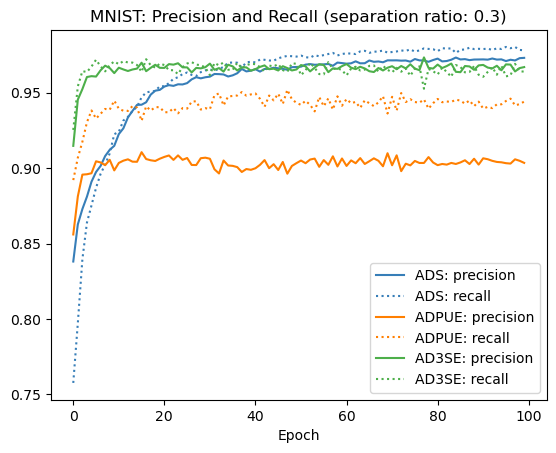}}
\subfigure{\includegraphics[width=0.3\linewidth,height=0.23\textwidth]{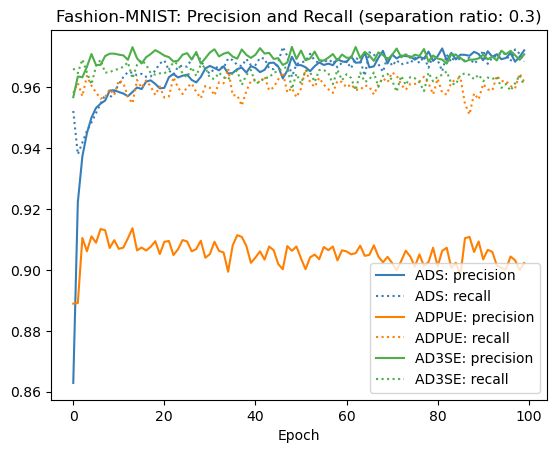}}
\subfigure{\includegraphics[width=0.3\linewidth,height=0.23\textwidth]{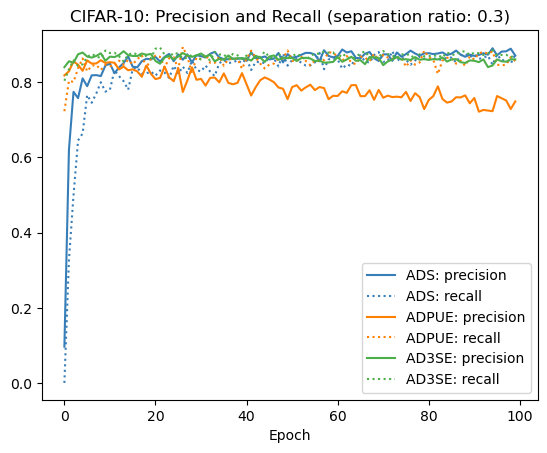}}
\subfigure{\includegraphics[width=0.3\linewidth,height=0.23\textwidth]{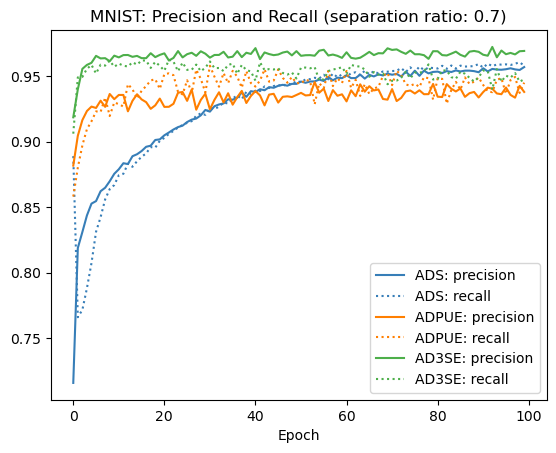}}
\subfigure{\includegraphics[width=0.3\linewidth,height=0.23\textwidth]{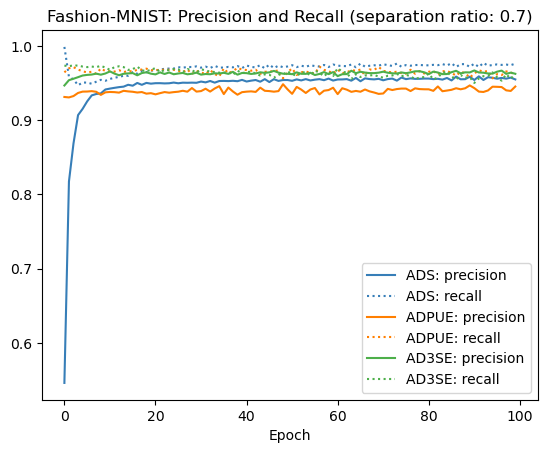}}
\subfigure{\includegraphics[width=0.3\linewidth,height=0.23\textwidth]{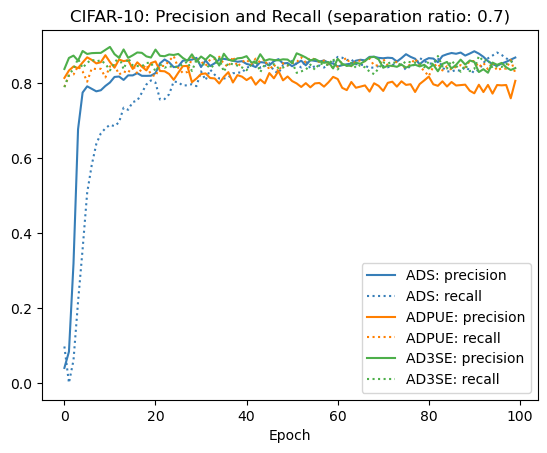}}
\vspace{-0.2cm}
\caption{Results of neural networks experiments. These plots show the precision and the recall of training epochs for two different separation ratio and three datasets. (Upper: separation ratio = $0.3$, Lower: separation ratio = $0.7$. Left: \texttt{MNIST}, Center: \texttt{Fashion-MNIST}, Right: \texttt{CIFAR-10})}
\label{fig:precision_and_recall}
\end{figure}

\section{Conclusion}
We formulated the PU classification problem with a selection bias and exposure labels as the PUE classification problem. We developed an identification strategy to address the selection bias problem under the strong ignorability assumption. By employing this assumption and strategy, we proposed an automatic debias learning method called ADPUE learning. In addition, we showed several additional problem settings and methods related to the problems, including the semi-supervised classification problem. In experiments, we confirmed the soundness of our proposed methods.

\bibliographystyle{icml2023}
\bibliography{arXiv.bbl}

\clearpage

\onecolumn

\appendix

\section{Proof of Lemma~\ref{lem:2}}
\label{appdx:lem:proof}
\begin{proof}
Because we minimize the functional,
\begin{align*}
&\widetilde{R}\left(f; f^\dagger\right)= \int\Big(\left(q(w = 1| \bm{x}) +  f^\dagger(\bm{x})\theta(e = 0| \bm{x})\right) \log(f(\bm{x}))\nonumber\\
&+ \left(q(w = 0| \bm{x}) -  f^\dagger(\bm{x})\theta(e = 0| \bm{x})\right) \log(1 - f(\bm{x}))\Big)\zeta(x)\mathrm{d}\bm{x},
\end{align*}
over all functions taking values in $(0, 1)$, the optimization is reduced to a point-wise minimization of 
\begin{align*}
&\ell(\beta) := \left(q(w = 1| \bm{x}) +  f^\dagger(\bm{x})\theta(e = 0| \bm{x})\right) \log(\beta)\nonumber\\
&\ + \left(q(w = 0| \bm{x}) -  f^\dagger(\bm{x})\theta(e = 0| \bm{x})\right) \log(1 - \beta)\Big)\zeta(x)\mathrm{d}\bm{x}.
\end{align*}
by considering $\beta$ as the optimization variable subject to $\beta \in (0, 1)$, which corresponds to $f(\bm{x})$ given $\bm{x}$. The derivative $\dot{\ell}(\beta)$ with respect to $\beta$ is given as
\begin{align*}
    \dot{\ell}(\beta)=&\frac{q(w = 1| \bm{x}) +  f^\dagger(\bm{x})\theta(e = 0| \bm{x})}{\beta}\ \ \ - \frac{q(w = 0| \bm{x})-  f^\dagger(\bm{x})\theta(e = 0| \bm{x})}{1-\beta}.
\end{align*}
Then, for the optimizer $\beta^*$, $\dot{\ell}(\beta^*) = 0$ is the first order condition of the optimization problem, and the optimizer gives as $\beta^* = q(w = 1| \bm{x}) +  f^\dagger(\bm{x})\theta(e = 0| \bm{x})$. 
For each $\bm{x}\in\mathcal{X}$, we obtain the optimizer. 
\end{proof}

\section{Details of Experiments}
\label{appdx:det_exp}
\subsection{Dataset information}
Here, we describe information of each dataset. We first summarize information of datasets used in experiments with linear models in Table~\ref{tbl:exp:dataset}. 

\begin{table}[h]
\caption{Datasets (Pos.frac.: Positive fraction).}
\label{tbl:exp:dataset}
\vspace{-0.5cm}
\begin{center}
\scalebox{0.8}[0.8]{\begin{tabular}{|c|ccc|}
\hline
Dataset&\# of samples&Pos. frac.&Dimension\\
\hline
{\tt australian} & 690 & 0.445 &  14 \\
{\tt w8a} & 49,749 (only use 1,800) & 0.589 &  300 \\
{\tt covtype} & 581,012 (only use 1,800)  & 0.438 &  784 \\
{\tt mushrooms} & 8124 (only use 1,800) & 0.878 &  112 \\
{\tt german}      & 1,000 & 0.300 &  24 \\
\hline
\end{tabular}}
\end{center}
\vspace{-0.5cm}
\end{table}

Next, we explain datasets used in experiments with neural networks as follows:
\begin{description}
    \item[\texttt{MNIST}] This dataset contains $70,000$ ($28\times 28$) grayscale images, each of which displays a single handwritten digit from 0 to 9. We label digit 1, 3, 5, 7, 9 with a negative value 0 and the others with a positive value 1.
    \item[\texttt{Fashion-MNIST}] This dataset contains $70,000$ ($28\times 28$) grayscale images, each of which with a label from 10 classes. 'T-shirt/top', 'Pullover', 'Coat', 'Shirt', 'Bag' are labeled with a positive value 1 and the others are labeled with 0.
    \item[\texttt{CIFAR-10}] This dataset consists of $60,000$ ($32\times 32$) color images in 10 classes. In our experiment, ‘bird’, ‘cat’, ‘deer’, ‘dog’, ‘frog’, ‘horse’ are labeled with negative value 0, while ‘airplane’, ‘automobile’, ‘ship’, ‘truck’ are labeled with positive value 1.
\end{description}

\subsection{Network structure}
In experiments with neural networks, for the \texttt{MNIST} and \texttt{Fashion-MNIST} datasets, we use a 3-layer \textit{multilayer perceptron} (MLP) with ReLU network structure model ($784 \times 100 \times 1$). For the \texttt{CIFAR-10} dataset, we use a LeNet-type CNN model \citep{LeCun1998}: $(32 \times 32\times 3)$-$C(3 \times 6, 5)$-$C(6 \times 16, 5)$-$400$-$120$-$84$-$1$, where the input is a $32 \times 32$ RGB image, $C(3\times
3, 96)$ means $96$ channels of $3\times3$ convolutions followed by a ReLU activation and a ($2\times 2$) max pooling. We train all models with Adam for $100$ epochs, the learning rate is $0.001$.

\end{document}

%% file: def.tex
\usepackage[utf8]{inputenc}
\usepackage[T1]{fontenc}
\usepackage{booktabs}
\usepackage{bm}
\usepackage{bbm}
\usepackage{tcolorbox}
\usepackage{multicol,multirow}
\usepackage{natbib}
\usepackage{comment}
\usepackage{booktabs}

\usepackage{enumerate}   
\usepackage{amsmath}
\usepackage{tikz}

\usetikzlibrary{positioning,arrows}

\usepackage{wrapfig}

\usepackage{amsmath}
\usepackage{appendix}
\usepackage{amssymb}

\usepackage{amsthm}
\usepackage{hyperref}

\newcommand{\iid}{ \stackrel{\mathrm{i.i.d.}}{\sim} }

\newcommand{\pa}{\mathrm{\pa}}

\usepackage{amsmath,amsfonts,bm}

\def\eqref#1{equation~\ref{#1}}
\def\1{\bm{1}}

\DeclareMathAlphabet{\mathsfit}{\encodingdefault}{\sfdefault}{m}{sl}
\SetMathAlphabet{\mathsfit}{bold}{\encodingdefault}{\sfdefault}{bx}{n}

\DeclareMathOperator*{\argmax}{arg\,max}
\DeclareMathOperator*{\argmin}{arg\,min}

%% file: arXiv.bbl
\begin{thebibliography}{37}
\providecommand{\natexlab}[1]{#1}
\providecommand{\url}[1]{\texttt{#1}}
\expandafter\ifx\csname urlstyle\endcsname\relax
  \providecommand{\doi}[1]{doi: #1}\else
  \providecommand{\doi}{doi: \begingroup \urlstyle{rm}\Url}\fi

\bibitem[Angrist \& Pischke(2008)Angrist and Pischke]{angrist_mostly_2008}
Angrist, J.~D. and Pischke, J.-S.
\newblock \emph{Mostly Harmless Econometrics: An Empiricist's Companion}.
\newblock 2008.

\bibitem[Bekker \& Davis(2018)Bekker and Davis]{Bekker2018}
Bekker, J. and Davis, J.
\newblock Learning from positive and unlabeled data under the selected at
  random assumption.
\newblock In \emph{Proceedings of the Second International Workshop on Learning
  with Imbalanced Domains: Theory and Applications}, volume~94, pp.\  8--22,
  2018.

\bibitem[Bekker \& Davis(2020)Bekker and Davis]{Bekker2020}
Bekker, J. and Davis, J.
\newblock Learning from positive and unlabeled data: a survey.
\newblock \emph{Machine Learning}, 109\penalty0 (4):\penalty0 719--760, 2020.

\bibitem[Bekker et~al.(2019)Bekker, Robberechts, and Davis]{Bekker2019}
Bekker, J., Robberechts, P., and Davis, J.
\newblock Beyond the selected completely at random assumption for learning from
  positive and unlabeled data.
\newblock In \emph{KDD}, pp.\  71–85, 2019.

\bibitem[Chen et~al.(2018)Chen, Feng, Ester, Zhou, Chen, and Wang]{Jiawei2018}
Chen, J., Feng, Y., Ester, M., Zhou, S., Chen, C., and Wang, C.
\newblock Modeling users' exposure with social knowledge influence and
  consumption influence for recommendation.
\newblock In \emph{CIKM}, 2018.

\bibitem[du~Plessis et~al.(2015)du~Plessis, Niu, and Sugiyama]{duPlessis2015}
du~Plessis, M.~C., Niu, G., and Sugiyama, M.
\newblock Convex formulation for learning from positive and unlabeled data.
\newblock In \emph{ICML}, pp.\  1386--1394, 2015.

\bibitem[Elkan \& Noto(2008)Elkan and Noto]{elkan2008learning}
Elkan, C. and Noto, K.
\newblock Learning classifiers from only positive and unlabeled data.
\newblock In \emph{KDD}, pp.\  213--220, 2008.

\bibitem[Harel(1979)]{Harel79}
Harel, D.
\newblock \emph{First-Order Dynamic Logic}, volume~68 of \emph{Lecture Notes in
  Computer Science}.
\newblock Springer-Verlag, 1979.

\bibitem[Heckman et~al.(1997)Heckman, Ichimura, and Todd]{Heckman1998}
Heckman, J.~J., Ichimura, H., and Todd, P.~E.
\newblock Matching as an econometric evaluation estimator: Evidence from
  evaluating a job training programme.
\newblock \emph{The Review of Economic Studies}, 64\penalty0 (4):\penalty0
  605--654, 1997.

\bibitem[Horvitz \& Thompson(1952)Horvitz and Thompson]{Horvitz1952}
Horvitz, D.~G. and Thompson, D.~J.
\newblock A generalization of sampling without replacement from a finite
  universe.
\newblock \emph{Journal of the American Statistical Association}, 1952.

\bibitem[Hsieh et~al.(2019)Hsieh, Niu, and Sugiyama]{Hsieh2019}
Hsieh, Y.-G., Niu, G., and Sugiyama, M.
\newblock Classification from positive, unlabeled and biased negative data.
\newblock In \emph{ICML}, volume~97, pp.\  2820--2829, 2019.

\bibitem[Hu et~al.(2008)Hu, Koren, and Volinsky]{Yifan2008}
Hu, Y., Koren, Y., and Volinsky, C.
\newblock Collaborative filtering for implicit feedback datasets.
\newblock In \emph{ICDM}, pp.\  263--272, 2008.

\bibitem[Imbens \& Rubin(2015)Imbens and Rubin]{imbens_rubin_2015}
Imbens, G.~W. and Rubin, D.~B.
\newblock \emph{Causal Inference for Statistics, Social, and Biomedical
  Sciences: An Introduction}.
\newblock 2015.

\bibitem[Jannach et~al.(2018)Jannach, Lerche, and Zanker]{Jannach2018}
Jannach, D., Lerche, L., and Zanker, M.
\newblock \emph{Recommending Based on Implicit Feedback}, pp.\  510--569.
\newblock 2018.

\bibitem[Joachims \& Swaminathan(2016)Joachims and Swaminathan]{Joachims2016}
Joachims, T. and Swaminathan, A.
\newblock Counterfactual evaluation and learning for search, recommendation and
  ad placement.
\newblock In \emph{SIGIR}, pp.\  1199–1201, 2016.

\bibitem[Joachims et~al.(2017)Joachims, Swaminathan, and
  Schnabel]{Joachims2017}
Joachims, T., Swaminathan, A., and Schnabel, T.
\newblock Unbiased learning-to-rank with biased feedback.
\newblock In \emph{WSDM}, pp.\  781–789, 2017.

\bibitem[Kato \& Teshima(2021)Kato and Teshima]{Kato2021bregman}
Kato, M. and Teshima, T.
\newblock Non-negative bregman divergence minimization for deep direct density
  ratio estimation.
\newblock In \emph{ICLR}, volume 139, pp.\  5320--5333, 2021.

\bibitem[Kato et~al.(2019)Kato, Teshima, and Honda]{kato2018pubp}
Kato, M., Teshima, T., and Honda, J.
\newblock Learning from positive and unlabeled data with a selection bias.
\newblock In \emph{ICLR}, 2019.

\bibitem[Kiryo et~al.(2017)Kiryo, Niu, du~Plessis, and Sugiyama]{Kiryo2017}
Kiryo, R., Niu, G., du~Plessis, M.~C., and Sugiyama, M.
\newblock Positive-unlabeled learning with non-negative risk estimator.
\newblock In \emph{NeurIPS}, pp.\  1675--1685, 2017.

\bibitem[Lecun et~al.(1998)Lecun, Bottou, Bengio, and Haffner]{LeCun1998}
Lecun, Y., Bottou, L., Bengio, Y., and Haffner, P.
\newblock Gradient-based learning applied to document recognition.
\newblock \emph{Proceedings of the IEEE}, 86\penalty0 (11):\penalty0
  2278--2324, 1998.

\bibitem[Liang et~al.(2016)Liang, Altosaar, Charlin, and Blei]{Liang2016}
Liang, D., Altosaar, J., Charlin, L., and Blei, D.~M.
\newblock Factorization meets the item embedding: Regularizing matrix
  factorization with item co-occurrence.
\newblock In \emph{RecSys}, pp.\  59–66, 2016.

\bibitem[Liu et~al.(2019)Liu, Lin, Zhang, Xiao, and Tong]{Liu2019}
Liu, D., Lin, C., Zhang, Z., Xiao, Y., and Tong, H.
\newblock Spiral of silence in recommender systems.
\newblock In \emph{WSDM}, pp.\  222–230, 2019.

\bibitem[Lu et~al.(2020)Lu, Zhang, Niu, and Sugiyama]{Nan2020}
Lu, N., Zhang, T., Niu, G., and Sugiyama, M.
\newblock Mitigating overfitting in supervised classification from two
  unlabeled datasets: A consistent risk correction approach.
\newblock In \emph{AISTATS}, volume 108, pp.\  1115--1125, 2020.

\bibitem[Luo et~al.(2021)Luo, Zhao, Chen, Qiao, Du, Zhang, Wu, Cai, He,
  Rajmohan, and Lin]{Zhao2021}
Luo, C., Zhao, P., Chen, C., Qiao, B., Du, C., Zhang, H., Wu, W., Cai, S., He,
  B., Rajmohan, S., and Lin, Q.
\newblock Pulns: Positive-unlabeled learning with effective negative sample
  selector.
\newblock \emph{AAAI}, 35\penalty0 (10):\penalty0 8784--8792, 2021.

\bibitem[Manski(2008)]{Manski2008}
Manski, C.
\newblock \emph{Partial Identification in Econometrics}, pp.\  1--9.
\newblock 2008.

\bibitem[Marlin \& Zemel(2009)Marlin and Zemel]{Benjamin2009}
Marlin, B.~M. and Zemel, R.~S.
\newblock Collaborative prediction and ranking with non-random missing data.
\newblock In \emph{RecSys}, pp.\  5--12, 2009.

\bibitem[Niu et~al.(2016)Niu, du~Plessis, Sakai, Ma, and Sugiyama]{Gang2016}
Niu, G., du~Plessis, M.~C., Sakai, T., Ma, Y., and Sugiyama, M.
\newblock Theoretical comparisons of positive-unlabeled learning against
  positive-negative learning.
\newblock In \emph{NeurIPS}, volume~29, 2016.

\bibitem[Rosenbaum \& Rubin(1983)Rosenbaum and Rubin]{Rosenbaum1983}
Rosenbaum, P.~R. and Rubin, D.~B.
\newblock {The central role of the propensity score in observational studies
  for causal effects}.
\newblock \emph{Biometrika}, 70\penalty0 (1):\penalty0 41--55, 1983.

\bibitem[Rubin(1978)]{Rubin1978}
Rubin, D.~B.
\newblock {Bayesian Inference for Causal Effects: The Role of Randomization}.
\newblock \emph{The Annals of Statistics}, 6\penalty0 (1):\penalty0 34 -- 58,
  1978.

\bibitem[Saito et~al.(2020)Saito, Yaginuma, Nishino, Sakata, and
  Nakata]{Saito2020}
Saito, Y., Yaginuma, S., Nishino, Y., Sakata, H., and Nakata, K.
\newblock Unbiased recommender learning from missing-not-at-random implicit
  feedback.
\newblock In \emph{WSDM}, pp.\  501–509, 2020.

\bibitem[Sakai et~al.(2017)Sakai, du~Plessis, Niu, and Sugiyama]{Sakai2017}
Sakai, T., du~Plessis, M.~C., Niu, G., and Sugiyama, M.
\newblock Semi-supervised classification based on classification from positive
  and unlabeled data.
\newblock In \emph{ICML}, volume~70, pp.\  2998--3006, 2017.

\bibitem[Sakai et~al.(2018)Sakai, Niu, and Sugiyama]{Sakai2018}
Sakai, T., Niu, G., and Sugiyama, M.
\newblock Semi-supervised auc optimization based on positive-unlabeled
  learning.
\newblock \emph{Machine Learning}, 107\penalty0 (4):\penalty0 767--794, 2018.

\bibitem[Schnabel et~al.(2016{\natexlab{a}})Schnabel, Swaminathan, Singh,
  Chandak, and Joachims]{Schnabel2016}
Schnabel, T., Swaminathan, A., Singh, A., Chandak, N., and Joachims, T.
\newblock Recommendations as treatments: Debiasing learning and evaluation.
\newblock In \emph{ICML}, pp.\  1670–1679, 2016{\natexlab{a}}.

\bibitem[Schnabel et~al.(2016{\natexlab{b}})Schnabel, Swaminathan, Singh,
  Chandak, and Joachims]{pmlr-v48-schnabel16}
Schnabel, T., Swaminathan, A., Singh, A., Chandak, N., and Joachims, T.
\newblock Recommendations as treatments: Debiasing learning and evaluation.
\newblock In \emph{ICML}, volume~48, pp.\  1670--1679, 2016{\natexlab{b}}.

\bibitem[Sugiyama et~al.(2022)Sugiyama, Bao, Ishida, Lu, and
  Sakai]{Sugiyama2022}
Sugiyama, M., Bao, H., Ishida, T., Lu, N., and Sakai, T.
\newblock \emph{Machine Learning from Weak Supervision: An Empirical Risk
  Minimization Approach (Adaptive Computation and Machine Learning series)}.
\newblock 8 2022.

\bibitem[Wang et~al.(2018{\natexlab{a}})Wang, Zheng, Yang, and Zhang]{Wang2018}
Wang, M., Zheng, X., Yang, Y., and Zhang, K.
\newblock Collaborative filtering with social exposure: A modular approach to
  social recommendation.
\newblock In \emph{AAAI}, 2018{\natexlab{a}}.

\bibitem[Wang et~al.(2018{\natexlab{b}})Wang, Golbandi, Bendersky, Metzler, and
  Najork]{Xuanhui2918}
Wang, X., Golbandi, N., Bendersky, M., Metzler, D., and Najork, M.
\newblock Position bias estimation for unbiased learning to rank in personal
  search.
\newblock In \emph{WSDM}, pp.\  610–618, 2018{\natexlab{b}}.

\end{thebibliography}
